\documentclass[12pt]{elsarticle}

\usepackage{amssymb}

\usepackage{amsmath}
\usepackage{graphicx}
\usepackage{verbatim}
\usepackage{amssymb}
\usepackage{graphicx}
\usepackage{rotating}
\usepackage{lscape}
\usepackage{subfig}
\usepackage{algorithm}
\usepackage[noend]{algpseudocode}
\usepackage{amsthm}
\usepackage{moresize}

\newtheorem{theorem}{Theorem}
\newtheorem{lemma}[theorem]{Lemma}

\begin{document}

\begin{frontmatter}



\title{Goal Seeking Quadratic Unconstrained Binary Optimization}

\author{Amit Verma\corref{cor1}}
\ead{averma@missouriwestern.edu}
\cortext[cor1]{Corresponding author}

\author{Mark Lewis}

\address{Craig School of Business, Missouri Western State University, Saint Joseph, MO, 64507, United States}

\begin{abstract}
	
The Quadratic Unconstrained Binary Optimization (QUBO) modeling and solution framework is a requirement for quantum and digital annealers. However optimality for QUBO problems of any practical size is extremely difficult to achieve. In order to incorporate the problem-specific insights, a diverse set of solutions meeting an acceptable target metric or goal is the preference in high level decision making. In this paper, we present two alternatives for goal-seeking QUBO for minimizing the deviation from a given target as well as a range of values around a target. Experimental results illustrate the efficacy of the proposed approach over Constraint Programming for quickly finding a satisficing set of solutions.

\end{abstract}

\begin{keyword}
Quadratic Unconstrained Binary Optimization, pseudo-Boolean optimization, goal seeking, what if, goal programming, quantum computer
\end{keyword}

\end{frontmatter}


\section{Introduction}
Quadratic Unconstrained Binary Optimization (QUBO) is a popular modeling framework wherein many combinatorial optimization problems can be recast into the form $min \; x'Qx; x \in \{0,1\}$ where $Q$ is a symmetric matrix of size $n$ with integer or real components and $x$ is a binary vector (see \cite{kochenberger2014unconstrained} for a survey of the broad applicability of this modeling structure). Optimization software such as CPLEX and Gurobi find an exact solution that optimizes the objective function value. Constraint Programming (CP) finds feasible solutions satisfying a set of constraints. Thus, CP supports a decision maker who is interested in setting goals and generating a set of solutions to consider for implementation, for example, a target time to complete a task, multi-criteria profit, risk, or time-frame targets. In the mixed integer linear programming paradigm, the objective target is modeled as lower and upper bound constraints on the objective function.  However for quadratic binary objectives, these constraints are also binary quadratic and are very difficult for conventional linear solvers such as CPLEX and Gurobi. In this short paper, we present two variants of a goal-seeking or target QUBO. In the first case, the target $t$ is exact and we are interested in finding binary vectors $x$ such that $x'Qx = t$. In the second case, the target is uncertain but defined by an interval $[lb,ub]$.

Target QUBO belongs to the class of multi-criteria goal programming where the objective function $x'Qx$ is provided a goal or target to be achieved and deviation from the target may be seen as the quadratic Taguchi quality loss function (\cite{taguchi1993taguchi}), although other measures such as absolute value are valid. Our implementation utilizes the squared deviation from the goal, which is minimized in the tabu-search heuristic. However, the importance of the magnitude of deviation from the target is typically determined by the decision-maker.

Our proposed model finds utility in four scenarios. First, the settings where we are interested in finding the input parameters that achieve a desired target level benefit from this analysis which is closely related to the ``goal-seeking problem" defined in \cite{arsham1998algorithms} and exemplified in the parameter selection method of Taguchi (\cite{arsham1997goal}). For instance, near-optimal solutions in RNA folding prediction are often better predictors of RNA structure than the optimal, and the near-optimal solutions occur in the neighborhood around a target energy level (\cite{lewis2021qfold}). Second, in sensitivity analysis of a discrete event simulation (see \cite{arsham1998algorithms} for more details). The third application arises in the context of the Constraint Satisfaction Problem modeled as a QUBO wherein a target number of constraints are specified for satisfaction.


Fourth, target QUBO can be utilized as a tool to solve multi-objective optimization problems. For instance, when the model appears as a subproblem for lexicographic optimization and the objectives are sorted from most important to least important (\cite{marques2011boolean}). Next, we find the binary variables satisfying the inequality for the $i'th$ objective $x'Q_ix \leq f_i^* + \delta_i$ which guarantees that the search focuses on near-optimal function evaluation $f_i^*$ controlled by a tolerance level $\delta_i$ which translates to $x'Q_ix \in [f_i^*,f_i^*+\delta_i]$ in our modeling framework. Thus, the results are useful for lexicographic optimization wherein we are looking for solutions that are close to the optimal value. Moreover, various objectives can be combined into a weighted multi-objective function. Thus, the various solution techniques proposed for multi-objective QUBO (like \cite{liefooghe2014hybrid,zangari2017decomposition,zhou2018multi,zhou2019ensemble}) benefit from our analysis.

The paper is organized as follows. The goal-seeking QUBO model is presented in Section 2. The solution technique and the results are presented in Section 3 and 4, followed by conclusions and future research directions.

\section{Model}

First, we present the exact goal-seeking QUBO for a deterministic target $t$. An exact model with quadratic penalties minimizing the squared deviation from the target is $min \; (x'Qx - t)^2 = min \; (x'Qx)^2 - 2t(x'Qx) + t^2$ wherein the first term has fourth-degree pseudo boolean functions of the form $x_i x_j x_k x_l$ significantly increasing the computational complexity. We could iteratively use Rosenberg quadratization substituting the quadratic subterm $x_i x_j$ with an additional binary variable $z_{ij}$ thereby introducing a penalty term $M(x_i x_j - 2 x_i z_{ij} - 2 x_j z_{ij} + 3 z_{ij})$ for a minimization problem. In this way, we could transform the fourth degree polynomial to a quadratic polynomial and utilize existing QUBO solvers. However, this reformulation leads to a large number of auxiliary variables and penalty terms (\cite{verma2020optimal}).

While the number of auxiliary variables required for the transformation of a fourth degree goal-seeking QUBO is O($n^2$), calculating the impact on the achievement function $(x'Qx - t)^2$ as part of a one-flip search routine is fast given that the incremental function evaluation of $x'Qx$ could be done in $O(1)$ (\cite{kochenberger2004unified}). Thus, selecting the best variable out of one-flip neighborhood is $O(n)$. Only binary vectors that meet the target are collected as part of the search routine and presented to the decision-maker. 

A common feature of integer optimization problems is the presence of alternate optima (\cite{hannan1980nondominance}) and while it may not be practical to enumerate all alternate optima, our solution approach generates a list of binary solution vectors meeting the target value or range. This feature is important to multi-criteria decision-making, wherein the decision-maker selects the preferred solution from the list of pareto optimal solution vectors.


The landscape theory provides us techniques to quantify the impact on the number of local optima due to width of the target interval. According to \cite{stadler2002fitness}, the number of local optima $M$ in a combinatorial landscape is given by the following expression:
\begin{equation}
	M \approx \frac{|X|}{|X(x_o,l)|}
\end{equation}

where $|X|$ denotes the total size of the solution space and $X(x_o,l)$ represents the set of solutions accessible from the vector $x_o$ in $l$ or less local moves. Note that the width of the interval around $x_o$ increases as $l$ increases. An estimate of $M$ for QUBO of size $n$ provided by \cite{chicano2013elementary} is given by:
\begin{equation}
	M \approx \frac{2^n}{\sum_{i=0}^{l} {n \choose i}}
\end{equation}

Using this estimate, we could derive the following related to Lemma 2 in \cite{verma2020penalty}:
\begin{lemma}
	As $l$ increases, the estimate of the number of local optima increases.
\end{lemma}
\begin{proof}
	We need to prove that the expression of $M$ provided in Equation (2) is increasing in $l$. According to \cite{lovasz2006diskrete}, the denominator is approximated by:
	\begin{equation}
		\sum_{i=0}^{l} {n \choose i} \leq 2^{n-1} exp {-\frac{(n-2l-2)^2}{4(1+l-n)}}
	\end{equation}
	Substituting this in Equation (2), we have:
	\begin{equation}
		M \approx 2 \; exp {-\frac{(n-2l-2)^2}{4(n-l-1)}}
	\end{equation}
	$M$ is increasing in $l$ as long as $log(M)$ is increasing in $l$. The first derivative of $log(M)$ is given by a positive multiple of $[n-(l+1)][n-(2l+2)][n-2l]$. Hence $M$ is increasing in $l$ as long as $n \ge 2l+2$. In other words, for a sufficiently large $n$ the lemma holds. Thus the number of local optima in the optimization problem increases as the size of the neighborhood increases. Hence, the number of combinations leading to a target also increases.
\end{proof}

Second, we describe the interval goal-seeking QUBO that seeks to find solution vectors that satisfy $x'Qx \in [lb,ub]$. This is attained by optimizing the achievement function $min \; (x'Qx - lb)(x'Qx - ub) = min \; (x'Qx)^2 -(lb+ub) (x'Qx) + lb*ub$. Again we face computational difficulties with solving this exact fourth-degree pseudo boolean function involving a large number of auxiliary variables and penalty terms.  Alternatively, lower and upper bound constraints on the objective function can be used, but most solvers do not work well with quadratic binary constraints. Thus we integrate the target constraint into the solution method as described in Section 3.

Because of the quadratic nature of the achievement function, the optimal value is achieved at $x'Qx = (lb+ub)/2$. However, in general the achievement function is non-positive as long as $x'Qx \in [lb,ub]$. Moreover, the binary vectors outside this interval lead to a higher penalty term, which increases with the squared distance from the interval bounds. This incentivizes the solver to limit the search to the solution landscape around the interval. Any binary vectors with non-positive function evaluation are easily collected during the search routine. Such vectors satisfying the interval constraint are presented to the decision-maker in descending order according to their objective function. 

The following simple lemma captures the relationship between satisficing vectors and achievement function for either of the variants of goal-seeking QUBO.
\begin{lemma}
	Achievement function $AF(x^*) \leq 0$ if and only if $x^*$ is a satisficing solution vector.
\end{lemma}
\begin{proof}
	The deterministic goal-seeking QUBO has $AF(x) = (x'Qx - t)^2$. Note that this quadratic function is always positive except for the case wherein $x'Qx = t$. Thus,  $AF(x) = 0$ only when the binary vector $x$ meets the target $t$.
	
	Any solution vector $x$ for the interval goal-seeking QUBO with $AF(x) = (x'Qx - lb)(x'Qx - ub)$ should have objective function evaluation $x'Qx$ belonging to one of the following three intervals: (i) $(-\infty,lb)$ (ii) $[lb,ub]$ (iii) $(ub,\infty)$. If $x'Qx \in (-\infty,lb)$, then $(x'Qx - lb)(x'Qx - ub) > 0$ since both of the subterms evaluate to negative values. Similarly, if $x'Qx \in (ub,\infty)$, then $(x'Qx - lb)(x'Qx - ub) > 0$ because both subterms evaluate to positive values. The only scenario wherein $AF(x)$ given by $(x'Qx - lb)(x'Qx - ub) \leq 0$ happens for $x'Qx \in [lb,ub]$ because of the conflicting signs of the two components.
\end{proof}

\section{Solution Technique}

Our solution method is summarized in Algorithm 1. We utilize a tabu-search (TS) heuristic for finding the binary vectors that meet the user-defined goals. A neighborhood move is defined by a one-flip i.e. setting $x_i$ to $1-x_i$ for a specific variable $i$. The best impact on the objective function $x'Qx$ due to one-flip is calculated in an incremental manner building on a previous solution requiring $O(n)$ (\cite{kochenberger2004unified}) to select the best one-flip from the $n$ variables in the neighborhood. We start from the all-zero solution and iteratively choose the non-tabu move that leads to the best improvement in the achievement function ($AF$). In case an improvement in $AF$ is not possible i.e., we reach a local optima, we chose the non-improving and non-tabu move having the least impact on $AF$. As discussed earlier, $AF$ is defined as $(x'Qx - t)^2$ for deterministic target and $(x'Qx - lb)(x'Qx - ub)$ for interval target. 


\begin{algorithm}
	\scriptsize
	\caption{Tabu search heuristic}
	\label{heur}
	\begin{algorithmic}[1] 
		\Procedure{Tabu Search}{$TS$} \Comment{Returns the set $S$ of satisficing solutions}
		\State Input: $Q, x_{initial}$ \Comment{$Q$ matrix and starting solution}
		\State Input: $TabuList$ \Comment{Tabu list of fixed length}
		\State Output: $S$ \Comment{Set containing output vectors satisfying target}
		\State $S \gets \emptyset$
		\State $x_{*} \gets x_{initial}$
		\While {termination criteria is not met}
			\State $best_{AF} \gets \infty$
			\For{each $x \in \; Neighborhood(x_{*})$}  \Comment{For all one flip neighbors of best solution}
			\If{$x \notin \; TabuList$} \Comment{Variable is non-tabu}
				\If {$AF(x) < best_{AF}$} \Comment{Solution vector improves achievement function}
					\State $x_{*} \gets x$
					\State $best_{AF} \gets AF(x)$
				\EndIf
			\EndIf
			\EndFor
			\State $TabuList \gets Update(TabuList)$ \Comment{Remove a variable from the tabu list}
			\State $TabuList \gets Update(TabuList, x_{*})$	\Comment{Add the best choice variable to the tabu list}
			\If{$AF(x_*) \leq 0$} \Comment{Solution vector meets the specified target}
			\State $S \gets S \bigcup x_*$ \Comment{Add the solution vector to the output set}
			\EndIf
		\EndWhile
		\State $S \gets Postprocess(S)$ \Comment{Remove the duplicates from output set $S$}	
		\State \textbf{return} $S$
		\EndProcedure
	\end{algorithmic}
\end{algorithm}

To help prevent cycling to previously visited solutions, we define a tabu list consisting of the tabu tenure of the variables that were flipped. Thus, a non-tabu variable $x_i$ is added to the tabu list with maximum tabu tenure if it was the best choice for one-flip. The variable is eventually removed from the tabu list after a fixed number of iterations set by the tabu tenure parameter. Note that this parameter has an impact on the computational performance of the heuristic. In general, as the tabu tenure increases, the chances for cycling reduces  while a too large tabu tenure parameter inhibits the ability to reach a target. 

Compared to a traditional heuristic wherein we are interested in finding the best solution in a limited runtime, we focus on finding multiple solutions that meet the user-defined objective goal. In this way, we can provide various choices to the decision-maker. The termination criteria in our TS heuristic is the solution runtime and our algorithm might find duplicate solutions due to cycling, so the end of the routine, we run a post-processing method that identifies the unique solutions among the solution set.



The solution quality of target QUBO is assessed by the number of unique solutions obtained by the solver in a limited runtime. Because the tabu tenure parameter has an impact on the solution quality, we conducted preliminary experiments to determine a good setting. For this purpose, we utilized the $2500$ node ORLIB instances (\cite{beasley1990or}). The deterministic target was set at $80\%$ of the best known solutions of the problems (obtained from \cite{wang2012path}). In general, smaller tabu tenures lead to quick solutions, but tend to cycle, and the solver is often stuck in local optima. On the other hand, large tabu tenure restricts the neighborhood visited by the solver. Based on preliminary experiments, we observed that tabu tenure between $10$ and $20$ provides good results on the ORLIB problems and we set the tabu tenure to $10$ for the problem set in the following section.

\section{Results}
For computational experiments, we use the QUBO instances with $2500, 3000$ and $4000$ variables (presented in \cite{beasley1990or} and \cite{palubeckis2004multistart}). The solution technique was implemented using C. We also utilize the IBM Constraint Programming (CP) Optimizer for benchmarking with the exact model. The exact non-linear model for the two variants namely $x'Qx = t$ and $x'Qx >= lb \; \& \; x'Qx <= ub$ was input to the CP solver using Python DOCPLEX API. The target level is set as a higher percentage of best known solutions ($bks$) namely $80\%,85\%,90\%$ and $95\%$. The experiments were performed on a 3.40 GHz Intel Core i7 processor with 16 GB RAM running 64 bit Windows 7 OS. 

In terms of benchmarking, the IBM CP solver was unable to find any feasible solution for either of the problem variants in $100$ seconds and for bigger instances, the size of the exact model was too large in terms of memory requirements. We also tried the default portfolio of solvers provided by MiniZinc (\cite{nethercote2007minizinc}) for the non-linear models, and those also yielded no solutions. Thus, these problems are computationally challenging.

First, we present the results using our heuristic for the exact target QUBO problem $x'Qx = t$ in Table 1. It is evident from the results that as we get closer to $bks$, the goal-setting problem becomes more challenging and we conclude that the solution landscape for these problems are distributed such that the number of feasible solutions decrease as the target approaches $bks$.


\begin{table}[htbp]
	\centering
	\caption{Number of unique solutions for $x'Qx = t$}
	\scalebox{0.85}{
	\begin{tabular}{|r|rrrr|r|rrrr}
		\hline
		\multicolumn{1}{l}{Instance} & 80\%  & 85\%  & 90\%  & 95\%  & \multicolumn{1}{l}{Instance} & 80\%  & 85\%  & 90\%  & 95\% \\
		\hline
		\textit{2500.1} & 31    & 27    & 22    & 15    & \textit{3000.1} & 28    & 22    & 19    & 17 \\
		\textit{2500.2} & 21    & 17    & 8     & 6     & \textit{3000.2} & 61    & 52    & 41    & 30 \\
		\textit{2500.3} & 30    & 25    & 19    & 16    & \textit{3000.3} & 60    & 54    & 45    & 36 \\
		\textit{2500.4} & 36    & 36    & 20    & 19    & \textit{3000.4} & 42    & 33    & 27    & 16 \\
		\textit{2500.5} & 32    & 31    & 30    & 22    & \textit{3000.5} & 31    & 26    & 18    & 14 \\
		\hline
		\textit{2500.6} & 26    & 24    & 14    & 8     & \textit{4000.1} & 38    & 30    & 24    & 14 \\
		\textit{2500.7} & 60    & 59    & 39    & 31    & \textit{4000.2} & 29    & 18    & 14    & 8 \\
		\textit{2500.8} & 27    & 23    & 16    & 16    & \textit{4000.3} & 33    & 28    & 24    & 12 \\
		\textit{2500.9} & 45    & 37    & 33    & 18    & \textit{4000.4} & 29    & 24    & 19    & 16 \\
		\textit{2500.10} & 33    & 23    & 19    & 8     & \textit{4000.5} & 52    & 50    & 44    & 27 \\
	\end{tabular}}%
	\label{tab:addlabel}%
\end{table}%

Second, we investigate the performance of our heuristic on interval targets i.e. $x'Qx \in [lb,ub]$ in Table 2. For this case, we choose the intervals as $(80\%,85\%)$,$(85\%,90\%)$,$(90\%,95\%)$ and $(95\%,100\%)$ of $bks$. Note that $(lb,ub)$ is equivalent to $[lb+1,ub-1]$ for functions with integral coefficients. We again observe from the results that the intervals close to $bks$ do not contain as many feasible solutions. Moreover, the number of unique feasible solutions found while satisfying an interval goal as well as the deterministic goal are dependent on the solution landscape which is problem specific. 

The diversity attributes for the set of solution vectors ($S$) output by the heuristic for a specific target were also analyzed. For this purpose, the hamming distance $h(x,y)$ between two solution vectors $x$ and $y$ given by the number of difference bits was computed for each pair of solution vectors. Thereafter, we calculated the mean hamming distance given by $\sum_{x,y \in S}^{x \neq y} h(x,y) / {|S| \choose 2}$. The denominator represents the number of unique solution pairs. The mean hamming distance for all output sets was between $1$ and $2$. The median hamming distance is $1$ or $2$ while the maximum hamming distance is $2$ or $3$ bits. Thus the solutions found were not diverse. In other words, alternate optima for a specific target near $bks$ are located in close proximity in the solution landscape. Based on this finding, it is more likely to find the next satisficing binary vector near the current one. 

\begin{table}[htbp]
	\centering
	\caption{Number of unique solutions for $x'Qx \in [lb,ub]$}
	\scalebox{0.8}{
	\begin{tabular}{|r|rrrr|r|rrrr}
		\hline
		\multicolumn{1}{l}{Instance} & (80,85)  & (85,90)  & (90,95)  & (95,100)  & \multicolumn{1}{l}{Instance} & (80,85)  & (85,90)  & (90,95)  & (95,100) \\
		\hline
		\textit{2500.1} & 29    & 22    & 14    & 10    & \textit{3000.1} & 19    & 15    & 13    & 10 \\
		\textit{2500.2} & 15    & 11    & 7     & 4     & \textit{3000.2} & 42    & 33    & 26    & 16 \\
		\textit{2500.3} & 23    & 23    & 16    & 10    & \textit{3000.3} & 42    & 37    & 27    & 19 \\
		\textit{2500.4} & 32    & 25    & 20    & 13    & \textit{3000.4} & 24    & 22    & 11    & 10 \\
		\textit{2500.5} & 30    & 29    & 17    & 13    & \textit{3000.5} & 17    & 17    & 14    & 8 \\
		\hline
		\textit{2500.6} & 19    & 14    & 9     & 7     & \textit{4000.1} & 25    & 20    & 12    & 6 \\
		\textit{2500.7} & 52    & 46    & 25    & 16    & \textit{4000.2} & 21    & 10    & 11    & 5 \\
		\textit{2500.8} & 18    & 17    & 15    & 13    & \textit{4000.3} & 18    & 16    & 14    & 8 \\
		\textit{2500.9} & 29    & 24    & 16    & 6     & \textit{4000.4} & 16    & 18    & 9     & 5 \\
		\textit{2500.10} & 22    & 15    & 14    & 5     & \textit{4000.5} & 36    & 29    & 26    & 11 \\
	\end{tabular}}%
	\label{tab:addlabel}%
\end{table}%

\section{Conclusions and Future Research} \label{conc}
We present a novel model for two variants of goal-seeking QUBO involving either deterministic or uncertain targets. A number of alternate satisficing solutions were obtained using a greedy one-flip tabu search routine. Testing on benchmark instances clearly established the efficacy of our approach compared to state-of-the-art CP solvers which were unable to find solutions to any of the problems in the test set.

Regarding more detailed future research, our technique is easily extendable to multiple targets. For instance, achieving target levels of $t_1$ or $t_2$ would require parallelized instances of the routine searching for the two target levels. These threads could operate independently. Moreover, two different objective goals $t_1$ and $t_2$ with priorities $w_1$ and $w_2$ could be handled by changing the achievement function to $w_1 (x'Q_1x - t_1)^2 + w_2 (x'Q_2x - t_2)^2$. In this way, we construct the efficient solution frontier by enumerating different choices of $w_1$ and $w_2$. Herein, the weights are determined by relative importance according to the decision-maker and is closely related to weighted or non-preemptive goal programming (\cite{sherali1983preemptive}).

\bibliographystyle{elsarticle-num} 
\bibliography{targetQ}






\end{document}